\PassOptionsToPackage{usenames,dvipsnames}{xcolor}
\documentclass[]{siamonline190516}

\usepackage[sort]{cite}

\usepackage{silence,dsfont}
\usepackage{todonotes}
\usepackage[listings,skins,breakable]{tcolorbox}
{%
\end{tcolorbox}\endgroup}

\usepackage{amsmath,amssymb,amsfonts,mathtools}

\usepackage{subcaption}

\usepackage{csl}

\usepackage{tikz}
\usetikzlibrary{patterns}
\usetikzlibrary{snakes}
\usetikzlibrary{arrows.meta}

\DeclarePairedDelimiter\set\{\}

\newcommand{\appropto}{\mathrel{\vcenter{
  \offinterlineskip\halign{\hfil$##$\cr
    \propto\cr\noalign{\kern2pt}\sim\cr\noalign{\kern-2pt}}}}}

\def\!#1{\mathcal{#1}}
\def\*#1{\boldsymbol{\mathbf{#1}}}
\def\|#1{\textnormal{#1}}
\def\##1{\mathfrak{#1}}

\def\norm#1{\big\lVert#1\big\rVert}

\DeclareMathOperator*{\argmin}{arg\,min}

\usetikzlibrary{shapes}
\tikzset{
  dot hidden/.style={},
  line hidden/.style={},
  dot colour/.style={dot hidden/.append style={color=#1}},
  dot colour/.default=black,
  line colour/.style={line hidden/.append style={color=#1}},
  line colour/.default=black
}

\usepackage{xparse}
\NewDocumentCommand{\drawdie}{O{}m}{%
\begin{tikzpicture}[x=1em,y=1em,radius=0.1,#1]
  \draw[rounded corners=0.5,line hidden] (0,0) rectangle (1,1);
  \ifodd#2
    \fill[dot hidden] (0.5,0.5) circle;
  \fi
  \ifnum#2>1
    \fill[dot hidden] (0.2,0.2) circle;
    \fill[dot hidden] (0.8,0.8) circle;
   \ifnum#2>3
     \fill[dot hidden] (0.2,0.8) circle;
     \fill[dot hidden] (0.8,0.2) circle;
    \ifnum#2>5
      \fill[dot hidden] (0.8,0.5) circle;
      \fill[dot hidden] (0.2,0.5) circle;
     \ifnum#2>7
       \fill[dot hidden] (0.5,0.8) circle;
       \fill[dot hidden] (0.5,0.2) circle;
      \fi
    \fi
  \fi
\fi
\end{tikzpicture}%
}  

\usepackage{enumitem}
\setlist[enumerate]{leftmargin=.5in}
\setlist[itemize]{leftmargin=.5in}

\ifpdf
\hypersetup{
  pdftitle={A Meta-learning Formulation of the Autoencoder Problem},
  pdfauthor={A.A. Popov, A. Sarshar, A. Chennault, and A. Sandu}
}
\fi

\usepackage{algorithmic} 

\Crefname{ALC@unique}{Line}{Lines}
\newcounter{myalg}
\AtBeginEnvironment{algorithmic}{\refstepcounter{myalg}}
\makeatletter
\@addtoreset{ALC@unique}{myalg}
\makeatother

\newsiamremark{remark}{Remark}
\newsiamremark{hypothesis}{Hypothesis}
\crefname{hypothesis}{Hypothesis}{Hypotheses}
\crefname{enumi}{Property}{Properties}
\Crefname{enumi}{Property}{Properties}
\newsiamthm{claim}{Claim}

\headers{A Meta-learning Formulation of the Autoencoder Problem}{A.A. Popov, A. Sarshar, A. Chennault, and A. Sandu}

\title{A Meta-learning Formulation of the Autoencoder Problem \\ for Non-linear Dimensionality Reduction\thanks{Submitted to the arXiv \today.
\funding{This work was supported by DOE through award ASCR DE-SC0021313, by NSF  through award  CDS\&E--MSS 1953113, and by the Computational Science Laboratory at Virginia Tech.}}}

\author{Andrey A. Popov\thanks{Computational Science Laboratory, Department of Computer Science, Virginia Tech, 
  (\email{apopov@vt.edu}, \email{sarshar@vt.edu}, \email{achennault@vt.edu}, \email{sandu@cs.vt.edu}, \url{https://csl.cs.vt.edu}).}
\and Arash Sarshar\footnotemark[2]
\and Austin Chennault\footnotemark[2]
\and Adrian Sandu\footnotemark[2]}

 \csltitle{A Meta-learning Formulation of the Autoencoder Problem for Non-linear Dimensionality Reduction}
 \cslauthor{Andrey A. Popov, Arash Sarshar, Austin Chennault, Adrian Sandu}
 \cslyear{21}
 \cslreportnumber{9}
 \cslemail{apopov@vt.edu, sarshar@vt.edu, achennault@vt.edu, sandu@cs.vt.edu} 
\begin{document}

\csltitlepage

\maketitle

\begin{abstract}
A rapidly growing area of research is the use of machine learning approaches such as autoencoders for dimensionality reduction of data and models in scientific applications. We show that the canonical formulation of autoencoders suffers from several deficiencies that can hinder their performance.
Using a meta-learning approach, we reformulate the autoencoder problem as a bi-level optimization procedure that explicitly solves the dimensionality reduction task. 
We prove that the new formulation corrects the identified deficiencies with canonical autoencoders, provide a practical way to solve it, and showcase the strength of this formulation with a simple numerical illustration.
\end{abstract}



\section{Introduction}
\label{sec:introduction}

Autoencoders are often employed to  perform dimensionality reduction of data~\cite{aggarwal2018neural,Goodfellow-et-al-2016} for many scientific tasks in fields such as bioinformatics~\cite{ding2005minimum, toronen1999analysis}, finance~\cite{ni2006recurrent}, dynamical systems~\cite{brunton2019data,chinesta2016model, guckenheimer2013nonlinear,strogatz2018nonlinear}, and fluids~\cite{gonzalez2018deep}, and have also shown promise for performing non-linear model order reduction~\cite{cooper2021investigation,kim2020fast, lee2020model,maulik2021reduced}.
The recent trend of theory-guided machine-learning~\cite{karpatne2017theory} has spurred additional interest in machine learning methods that can be applied in a robust manner to the aforementioned scientific tasks. 

Classical methods for dimensionality reduction such as principal component analysis~\cite{james2013introduction} (or proper orthogonal decomposition~\cite{popov2021multifidelity,sirovich1987turbulence1,wang2012proper}), dynamic mode decomposition~\cite{kutz2016dynamic}, and locally linear embedding~\cite{roweis2000nonlinear}, all achieve some sense of interpretability through the preservation of a linear structure.
Linear methods, while nice and interpretable, often cannot reduce a general problem to is intrinsic dimension~\cite{lee2007nonlinear}, making their use challenging for certain applications.  
Autoencoders, on the other hand, are fully non-linear, and thus have  potential to overcome this fundamental limit of (quasi-)linear methods.

The advent of automatic differentiation (AD) in machine learning~\cite{baydin2018automatic} has given rise to significantly faster development of novel neural-network architectures by allowing much faster prototyping. 
Neural-networks are not the only type of algorithm that has benefited from recent AD research. Differentiable programming~\cite{innes2019differentiable}, has allowed AD-empowered methods such as back-propagation to be applied to a large class of classical differentiable algorithms.
This flexibility has allowed more freedom in considering solutions to machine-learning problems, and is a motivation of this work.

Meta-learning, commonly regarded as the process of \textit{learning-to-learn}, is a new paradigm for training of neural networks and other machine learning models that relies on a two-level approach~\cite{dempe2002foundations,franceschi2018bilevel}: the lower-level focuses on finding solutions to different tasks, while the \textit{meta}-level configures the learning algorithm for improved generalization~\cite{Vilalta2002Jun,Hospedales2021May}. 
Meta-learning has found applications in many disciplines such as online learning of new data~\cite{pmlr-v97-finn19a,pmlr-v70-finn17a}, adaptive data selection (curriculum learning)~\cite{Bengio2009Jun,pmlr-v80-ren18a}, hyper-parameter optimization~\cite{franceschi2018bilevel,pmlr-v108-lorraine20a, franceschi2018bilevel}, and neural network architecture discovery~\cite{Stanley2019Jan,Chen_2019_ICCV}.
There is also existing literature on meta-learning of transformations of input data to create new features for classification and regression tasks~\cite{lee2007learning}, and to aid autoencoders~\cite{pmlr-v22-zhou12b}.

This work identifies several properties that classical methods for dimensionality reduction possess, but regular autoencoders lack, and provide a new formulation of autoencoders which includes these properties. We extend some of the ideas presented in~\cite{popov2021nlmfenkf}, where the new ``right-inverse'' property for autoencoders was identified, and was shown to significantly aid the construction of reduced dimension models for dynamical systems.

The new contributions of this work are as follows. We formalize deficiencies of the standard autoencoder formulation, namely that it does not satisfy the right-inverse property and it does not guarantee that the reconstruction is encodable; we believe these are strong arguments for the need to reformulate the autoencoder problem for dimensionality reduction. We cast autoencoding as a meta-learning problem~\cite{franceschi2018bilevel,Hospedales2021May}, where the lower-level task identifies the reduced dimension representation of each data point in our space of interest, implicitly defining the encoder, and
the meta-task computes the optimal decoder that reconstructs a full dimensional representation from the reduced dimensional one. The new formulation is named meta-autoencoder.
We provide a practical implementation of the meta-autoencoder that embeds Gauss-Newton iterations into the encoder.  This practical implementation  allows to fully solve the semi-stochastic bi-level optimization problem through the conventional ADAM~\cite{kingma2014adam} algorithm without any modification to the training loop. We illustrate this practical implementation and the benefits that it provides on the MNIST dataset~\cite{deng2012mnist}.

This paper is organized as follows. We first provide background about canonical autoencoders and discuss their shortcomings when applied to the dimensionality reduction problem in \cref{sec:motivation}. We then introduce the meta-learning formulation of the autoencoder problem in \cref{sec:reduced-dimension-autoencoders}, and approaches for its practical implementation. A numerical illustration that demonstrates the power of our approach is given in \cref{sec:numerical-illustration}. Finally, concluding remarks are drawn in \cref{sec:conclusions}.

\section{Background and Motivation}
\label{sec:motivation}

Consider a random variable $X$ that represents the distribution of $n$-dimensional high fidelity data over the support $\widehat{\mathbb{X}} \subset \mathbb{X}$ with $\dim(\mathbb{X}) = n$, where the set of all high fidelity data is $\widehat{\mathbb{X}}$ and $\mathbb{X}$ is the full space in which the data resides. 
The goal of dimensionality reduction is to map $X$ to some $r$-dimensional reduced representation $U$, having support $\widehat{\mathbb{U}} \subset \mathbb{U}$ with $\dim(\mathbb{U}) = r$, where $\mathbb{U}$ is the underlying dimensionally reduced space of which the dimensionally reduced data set $\widehat{\mathbb{U}}$ is a subset. The reduced data  $U$ contains the important features of high fidelity data $X$, but is (usually) of significantly smaller dimension, $ r\ll n$. This dimension reduction is achieved by an {\it encoder}, i.e., a function,
\begin{equation}\label{eq:encoder}
    \theta(X) = U,\quad \theta : \mathbb{X}  \to \mathbb{U}, \quad \theta \in \mathcal{C}^\infty,
\end{equation}
which is a mapping from the full space $\mathbb{X}$ to the dimensionally reduced space $\mathbb{U}$. When \eqref{eq:encoder} is applied to the set $\widehat{\mathbb{X}}$ of all \textit{encodable} data, the image $\widehat{\mathbb{U}} = \theta(\widehat{\mathbb{X}})$ is the set of all encodings.
A {\it decoder} is another function,
\begin{equation}\label{eq:decoder}
    \phi(U) = \widetilde{X},\quad \phi : \mathbb{U} \to  \mathbb{X}, \quad \phi \in \mathcal{C}^\infty,
\end{equation}
which produces a reconstruction $\widetilde{X}\in \widetilde{\mathbb{X}} \subset \mathbb{X}$ of $X$, where $\widetilde{\mathbb{X}} = \phi(\widehat{\mathbb{U}})$ is the reconstruction set that lives in the same underlying space $\mathbb{X}$ as the original data. 
For the remainder of this paper we make the assumption that the encoder and decoder are smooth functions of their input, which is an important assumption for many scientific applications, and avoids a discussion of degenerate cases. 
The encoder~\cref{eq:encoder} and decoder~\cref{eq:decoder} combined define the \textit{autoencoder}.

Previous work assumes that a sufficient condition for defining the autoencoder is that the reconstruction $\widetilde{X}$ approximately matches the original data $X$ \cite{aggarwal2018neural,Goodfellow-et-al-2016}.
Under this assumption, the problem of finding the optimal encoder and decoder is typically posed as the following optimization problem:
\begin{equation}\label{eq:canonical-autoencoder}
    \phi_*,\,\theta_*  = \argmin_{\phi\,\in\*\Phi,\,\theta\,\in\*\Theta} \mathbb{E}_{X}\!L\big(\phi\big(\theta(X)\big), X\big),
\end{equation}
where the cost function $\!L : \mathbb{X} \times \mathbb{X} \to \mathbb{R}$ is typically taken to be the mean squared error. The sets $\*\Phi$ and $\*\Theta$ define all possible decoders and encoders, respectively, for the given problem. For example these can be sets of realizable neural network functions with some underlying architecture, though the ideas presented in this work are not restricted to this.
For the remainder of this paper, we refer to \cref{eq:canonical-autoencoder} as the {\it canonical autoencoder.}

We start with an example illustrating why~\cref{eq:canonical-autoencoder} is not sufficient for performing accurate dimensionality reduction.
Consider first reducing the dimension of a handwritten sample of the digit `$6$', and then reconstructing it. We schematically represent this process as:
\begin{equation}\label{eq:encode-decode-a-six}
    \theta\left(\fbox{6}\right) = \begin{bmatrix}\,\drawdie{6}\,\\\wr\end{bmatrix},\quad     \phi\left(\begin{bmatrix}\,\drawdie{6}\,\\\wr\end{bmatrix}\right) = \widetilde{\fbox{6}}.
\end{equation}
The reconstruction $\widetilde{\fbox{6}}$ is not exact, but is recognizable as the original data $\fbox{6}$, under certain criteria. The reduced dimensional representation in \cref{eq:encode-decode-a-six},
\begin{equation}
  X = \fbox{6}  \xrightarrow{\quad\theta\quad}  U = \begin{bmatrix}\,\drawdie{6}\,\\\wr\end{bmatrix},
\end{equation}
captures some features in the reduced space. The six dots on the die, $\drawdie{6}$, represent the features describing the number six, with the wavy line, $\wr$, representing the features describing how it was drawn.

\begin{figure}[t]
\centering
\scalebox{1.4}{%
    \begin{tikzpicture}[scale=1.25]

    \fill [fill opacity=0.2, red,rounded corners=5, draw]
    (0.0,0) --
    (0.0,1.1) --
    (2.05,1.1) --
    (2.05, 0) -- cycle {};
    \node[text=red!70!black] at (0.3,0.55) {$\widehat{\mathbb{X}}$};
    
    \node at (0.9, 0.55) {$\fbox{6}$};

    \fill [fill opacity=0.2, ForestGreen,rounded corners=5, draw]
    (2.2,0) --
    (2.2,1.1) --
    (4.75,1.1) --
    (4.75,0) -- cycle {};
    \node[text=ForestGreen!70!black] at (3.45, 0.55) {$\mathbb{X}\mkern-5mu \setminus \mkern-5mu\widehat{\mathbb{X}}$};
    
    \node at (4.3, 0.55) {$\fbox{\parbox[c][\heightof{\textit{6}}]{\widthof{\textit{6}}}{$\,\triangleright$}}$};
  
    \fill [fill opacity=0.2, RoyalBlue,rounded corners=5, draw]
    (1.4,0.1) --
    (1.4,1) --
    (2.9,1) --
    (2.9,0.1) -- cycle {};
    \node[text=RoyalBlue!70!black] at (1.725,0.59) {$\widetilde{\mathbb{X}}$};
  
    \node at (2.55, 0.61) {$\widetilde{\fbox{6}}$};

    \fill [color=gray,pattern=north west lines, fill opacity=0.2, rounded corners=5, draw] (0.6,-1.45) -- (0.6,-0.5) -- (2.25, -0.5) -- (2.25, -1.45) -- cycle {};
    \node[text=gray!70!black] at (0.9, -1) {$\mathbb{U}$};
    
    \node at(1.725, -1) {$\begin{bmatrix}\,\drawdie{6}\,\\\wr\end{bmatrix}$};
    
    \fill [color=gray,pattern=north east lines, fill opacity=0.2, rounded corners=5, draw] (2.4,-1.45) -- (2.4,-0.5) -- (4.25, -0.5) -- (4.25, -1.45) -- cycle {};
    \node[text=gray!70!black] at (2.85, -1) {$\mathbb{R}^r\mkern-5mu \setminus \mkern-5mu\mathbb{U}$};
    
    \node at (3.7, -1) {$\begin{bmatrix}\,\drawdie{3}\,\\\angle\end{bmatrix}$};
    
    \draw[color=black,-{Stealth[scale=2]}] (0.95,0.3) -- (1.675, -0.55);
    \node at (1.15, -0.2) {$\theta$};
   
   \draw[color=black,-{Stealth[scale=2]}] (1.775, -0.55) -- (2.5, 0.3);
   \node at (2.25, -0.2) {$\phi$};
   
    \draw[color=black,-{Stealth[scale=2]}] (2.6, 0.3) -- (3.65, -0.55);
    \node at (2.95, -0.2) {$\theta$};
    
    \draw[color=black,-{Stealth[scale=2]}] (3.75, -0.55) -- (4.25, 0.3);
    \node at (4.1, -0.2) {$\phi$};

\end{tikzpicture}%
    }
\caption{An illustration of the worst-case scenario for the canonical autoencoder that is described by \cref{eq:encode-decode-a-six} and \cref{eq:right-inverse-violation-example}.}
\label{fig:six-to-triangle-illustration}
\end{figure}
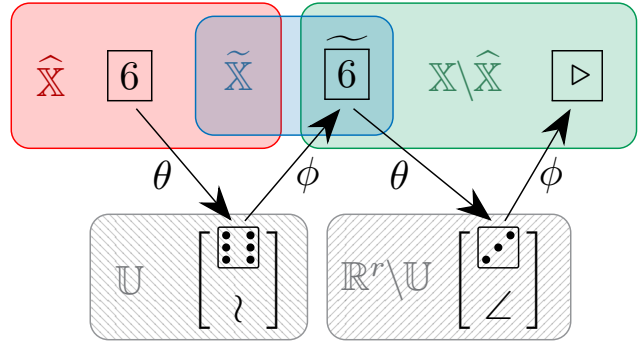

However, the representation of the reconstruction in the reduced space need not decode to a good approximation of the original data $\fbox{6}$. It may be the case that:
\begin{equation}\label{eq:right-inverse-violation-example}
    \theta\left(\widetilde{\fbox{6}}\right) = \begin{bmatrix}\,\drawdie{3}\,\\\angle\end{bmatrix}, \quad \phi\left(\begin{bmatrix}\,\drawdie{3}\,\\\angle\end{bmatrix}\right) = \fbox{\parbox[c][\heightof{\textit{6}}]{\widthof{\textit{6}}}{$\,\triangleright$}}.
\end{equation}
The reduced dimension representation of the reconstruction $\widetilde{\fbox{6}}$, could actually have features corresponding to the number three, $\drawdie{3}$, and a different, distinct, way of describing how to draw it, $\angle$. The corresponding reconstruction of said features could resemble a triangle $\fbox{\parbox[c][\heightof{\textit{6}}]{\widthof{\textit{6}}}{$\,\triangleright$}}$, which is clearly not recognizable as the digit `$6$'.
In other words, the reduced dimension representation of the reconstruction need not correspond to the reduced dimension representation of the original image.
A solution to~\cref{eq:canonical-autoencoder} may be unreliable, as for a given sample $X$, the reduced dimension representation $U$ is not necessarily the optimal reduced order representation of $X$, meaning that the reconstruction $\widetilde{X}$ need not represent $X$. An illustration of this example is presented in \cref{fig:six-to-triangle-illustration}.

It is easy to overlook why this is the case. One implicit assumption made when defining the canonical encoder and decoder \cref{eq:canonical-autoencoder} can be trivially violated. Take the reduced set,
\begin{equation}
    \widehat{\mathbb{U}} = \theta(\widehat{\mathbb{X}}),
\end{equation}
which is the set of all reduced order representations of the full order set under the action of the encoder.
Under the action of the decoder,
\begin{equation}
    \phi\big(\theta(\widehat{\mathbb{X}})\big) = \phi(\widehat{\mathbb{U}}) = \widetilde{\mathbb{X}},
\end{equation}
it is the reconstructed set $\widetilde{\mathbb{X}}$.
Notice that while the reconstruction set is a subset of full space $\mathbb{X}$, it is not at all necessary that it is a subset of the full data set at all, it could be that,
\begin{equation}
    \widetilde{\mathbb{X}} \not\subset \widehat{\mathbb{X}}.
\end{equation}
Therefore, there could potentially be reconstructed data that is not encodable.

As the pre-image of the reduced data set under the encoder is the full data set, 
\begin{equation}
    \theta^{-1}(\widehat{\mathbb{U}}) = \widehat{\mathbb{X}},
\end{equation}
it is possible that there exists an $X\in\widehat{\mathbb{X}}$, such that its reconstruction, $\widetilde{X}$ is not in the pre-image of the reduced data set under the encoder,
\begin{equation}
    \widetilde{X} = \phi\big(\theta(X)\big) \not\in \theta^{-1}(\widehat{\mathbb{U}}) = \widehat{\mathbb{X}},
\end{equation}
meaning that the action of the decoder would produce values that cannot themselves be encoded.
It therefore is the case that the reconstructed data set is not necessarily a subset of the pre-image of the reduced data set under the encoder,
\begin{equation}\label{eq:pre-image-violation}
    \widetilde{\mathbb{X}} = \phi(\widehat{\mathbb{U}}) \not\subset \theta^{-1}(\widehat{\mathbb{U}}) = \widehat{\mathbb{X}}.
\end{equation}
We call~\cref{eq:pre-image-violation} the \textit{pre-image violation}.

In the problem of reducing the dimension of a handwritten sample of the digit `6', it is desirable that the representation in the reduced space of the reconstruction is identical to the reduced space representation that the reconstruction was generated from,
\begin{equation}
    \theta\left(\widetilde{\fbox{6}}\right) = \begin{bmatrix}\,\drawdie{6}\,\\\wr\end{bmatrix},\quad     \phi\left(\begin{bmatrix}\,\drawdie{6}\,\\\wr\end{bmatrix}\right) = \widetilde{\fbox{6}},
\end{equation}
as the encoder would recognize the reconstruction as having the same reduced dimension representation as that which generated it in~\cref{eq:encode-decode-a-six}.
We call the satisfaction of the relationship,
\begin{equation}\label{eq:right-inverse-property}
    \theta\left(\phi\left(\theta\left(X\right)\right)\right) = \theta\left(X\right),\,\,\, \forall X \in \mathbb{X},
\end{equation}
the \textit{right-inverse property}, which sometimes is called the consistency property~\cite{abdulle2012heterogeneous}. It was previously shown in \cite{popov2021nlmfenkf} that weakly enforcing this property during training significantly improves the performance of autoencoder-based reduced order models.

The encoder-decoder pair generated by the canonical autoencoder~\cref{eq:canonical-autoencoder}, has the property that the cost function \cref{eq:canonical-autoencoder} is minimal over the totality of realizations of $X$ in the training data set $\widehat{\mathbb{X}}$. The formulation of the problem does not ensure that for any given realization of $X$ the corresponding realization of $U$ corresponds to a meaningful reconstruction. In other words, $U$ might not satisfy the right-inverse property~\cref{eq:right-inverse-property}.

We do not generally have access to the full data set $\widehat{\mathbb{X}}$, nor the distribution of $X$. In most cases $\widehat{\mathbb{X}}$ is uncountable, and we only have access to finite samples of the data, $\left[X_1, \dots X_N\right]$. In such a case issues such as overfitting~\cite{dietterich1995overfitting} arise. However, the problems that we discuss here are distinct from overfitting.
The pre-image violation in~\cref{eq:pre-image-violation} can occur even when the canonical autoencoder~\cref{eq:canonical-autoencoder} is trained \textit{over all possible} $\widehat{\mathbb{X}}$. This violation can cascade and cause a violation of the right inverse property~\cref{eq:right-inverse-property}.
In order to reconcile these facts, a new problem formulation for finding $\theta$ and $\phi$ is needed.

\begin{remark}
It is the authors' belief that conventional methods for `correcting' neural network-based autoencoders such as regularization~\cite{girosi1995regularization}, droupout~\cite{srivastava2014dropout}, and denoising~\cite{vincent2008extracting} among others, all contribute to alleviating the issues presented. The problem formulation presented in the following section is compatible with most, if not all, such methods. 
\end{remark}

\section{Formulation of Autoencoders as Meta-Learning Tasks}
\label{sec:reduced-dimension-autoencoders}

We construct a formulation of the autoencoder problem that addresses all three issue identified in \cref{sec:motivation}. Namely, one that explicitly accounts for the reduced dimension representation~\cref{eq:reduced-dimension-graph}, does not have the pre-image violation problem~\cref{eq:pre-image-violation}, and that satisfies the right-inverse property~\cref{eq:right-inverse-property}. 

The canonical autoencoder~\cref{eq:canonical-autoencoder} treats the dimensionality reduction task as a transformation from the full representation of $X$ to its reconstruction,
\begin{equation}
    X \xrightarrow{\quad\phi\,\circ\,\theta\quad} \widetilde{X},
\end{equation}
without explicitly taking into consideration the ``intermediate'' reduced dimension representation. However, this representation is an integral part of the actual dimensionality reduction task:
\begin{equation}\label{eq:reduced-dimension-graph}
    X \xrightarrow{\quad\theta\quad} U \xrightarrow{\quad\phi\quad} \widetilde{X}.
\end{equation}
We call~\cref{eq:reduced-dimension-graph} the \textit{reduced dimension representation graph}. Previous work \cite{Goodfellow-et-al-2016} has treated the canonical autoencoder problem~\cref{eq:canonical-autoencoder} as having the graphical representation~\cref{eq:reduced-dimension-graph}.

Following~\cref{eq:reduced-dimension-graph},  the reduced dimension representation problem can be decomposed into the following sub-problems: (i) given $X$,  find the reduced representation $U$, and (ii) given $U$, reconstruct $X$. 
As we only have access to the data $X$, the second sub-problem (ii) cannot be solved in isolation. \textit{Rather, we regard it as the following inverse problem: given $X$,  find a $U$ whose reconstruction resembles $X$.}

We now formalize this inverse problem as the \textit{dimensionality reduction task}.
Consider a smooth loss function $\!E(\widetilde{X}, X) : \mathbb{X} \times \mathbb{X} \to \mathbb{R}$ that quantifies the reconstruction error.
Assume that we are given a decoder~\cref{eq:decoder} function $\phi$. For some arbitrary data $X\in\mathbb{X}$, we want to find the optimal reduced dimension representation $U \in \mathbb{U}$ such that $\phi(U)$ reconstructs $X$. This problem is equivalent to finding a $U^{\phi, X} \in \mathbb{U}$ that minimizes $\!E$,
\begin{equation}\label{eq:dimensionality-reduction-task}
    U^{\phi, X} = \argmin_{U\, \in\, \mathbb{U}} \!E(\phi(U), X),
\end{equation}
where $\!E$  could be, but not necessarily is, the cost $\!L$ in~\cref{eq:canonical-autoencoder}. The superscript indicates that the optimal reduced order representation $U^{\phi, X}$ is defined with respect to a specific decoder $\phi$ and specific full representation $X$.

It is desirable that the dimensionality reduction task~\cref{eq:dimensionality-reduction-task} has a unique solution. We now show that under certain assumptions this is indeed the case.

\begin{theorem}\label{thm:unique-solution}
Let
\begin{equation}\label{eq:max-connected-set}
    \mathbb{B}^{\phi, X} = \left\{\mathbb{B} \mid \mathbb{B}\subset\mathbb{U} \text{ is max.-connected and }  \forall U \in \mathbb{B},\,\,  \!E(\phi(U), X) \text{ strictly convex in $U$}   \right\}
\end{equation}
define the set of all maximally-connected subsets of $\mathbb{U}$ such that $\!E(\phi(U), X)$ is strictly convex in the variable $U$ for a particular $\phi$ and $X$.

Assume there exists an oracle,
\begin{equation}\label{eq:convex-oracle}
   \mathfrak{O}^\phi : \mathbb{X} \to \mathbb{B}^{\phi, X},
\end{equation}
that maps each element of $X\in\mathbb{X}$ to some particular maximally-connected subset of $\mathbb{U}$ where $\!E(\phi(U), X)$ is convex in $U$ with the condition that,
\begin{equation}\label{eq:stable-set-oracle}
    \mathfrak{O}^\phi(X) = \mathfrak{O}^\phi\big(\phi(U^{\phi,X})\big),
\end{equation}
meaning that the reconstruction shares the same convex subset as the full representation.
Given any pair of decoder $\phi$ and full representation $X$, there exists a unique local minimum $U^{\phi,X}$ of $\!E(\phi(U), X)$ for any $X$ in some set in \cref{eq:max-connected-set}, such that the minimum for the reconstruction $\widetilde{X} \coloneqq \phi(U^{\phi,X})$, is identical to the minimum for the full data,
\begin{equation}\label{eq:full-and-reconstruction-reduced-shared}
    U^{\phi,\widetilde{X}} = U^{\phi,X}.
\end{equation}
\end{theorem}
\begin{proof}
If $\mathbb{B}^{\phi,X}$ contains at least one set,
then given some arbitrary $X\in\mathbb{X}$, there exists  $U^{\phi, X} \in \mathfrak{O}^\phi(X)$ such that either the gradient,
\begin{equation*}
    \left.\frac{\partial \!E(\phi(U), X)}{\partial U} \right\rvert_{U = U^{\phi, X}} = 0,
\end{equation*}
at that point, satisfies the first order optimality condition or $U^{\phi,X}$ is on the boundary of the set.
The Hessian is symmetric positive definite locally, on $\mathfrak{O}^\phi(X)$ by the convexity property,
\begin{equation*}
    \left.\frac{\partial^2 \!E(\phi(U), X)}{\partial U^2}\right\rvert_{U \in \mathfrak{O}^\phi(X)} > 0,
\end{equation*}
thus $U^{\phi, X}$ is a unique local minimum by the second order optimality condition~\cite{nocedal2006numerical}.

By~\cref{eq:stable-set-oracle}, the oracle defines the same set for each full dimension representation $X$ and corresponding  reconstruction $\widetilde{X}$, thus \cref{eq:full-and-reconstruction-reduced-shared} is satisfied, as required.
\end{proof}

\Cref{thm:unique-solution} does not guarantee that the global minimum of the dimensionality reduction task~\cref{eq:dimensionality-reduction-task} is found. It instead guarantees that, under its assumptions, it is possible to find a local minimum
\begin{equation}\label{eq:local-dimensionality-reduction-task}     
    U^{\phi, X} = \argmin_{U\, \in\, \mathfrak{O}^\phi(X) } \!E(\phi(U), X)
\end{equation}
within some subset of $\mathbb{U}$ given the oracle function~\cref{eq:convex-oracle}. This is a less ambitious goal, which we call the \textit{local dimensionality reduction task}.

If the conditions of \cref{thm:unique-solution} are satisfied, the local dimensionality reduction task~\cref{eq:local-dimensionality-reduction-task} provides a function from $X$ to $U^\phi$ that defines the optimal encoder,
\begin{equation}\label{eq:optimal-encoder}
    \theta^\phi(X) \coloneqq U^{\phi,X}, \quad \forall X \in \mathbb{X},
\end{equation}
implicitly in terms of the decoder $\phi$.
We now show that it is sufficient to define the  oracle~\cref{eq:convex-oracle} by an `initial condition function'.

\begin{corollary}\label{cor:initial-condition-function}
Given an `initial condition function' that maps some point $X$ onto a point on the subset of $\mathbb{U}$ defined by the oracle~\cref{eq:convex-oracle},
\begin{equation}\label{eq:encoder-initial-condition}
    \theta_0^\phi(X) = U_0^X \in \mathfrak{O}^\phi(X), \quad \theta_0^\phi : \mathbb{X} \to \mathbb{U},
\end{equation}
with the property that maps both a full order representation and its reconstruction to the same point~\cref{eq:right-inverse-property}, the property,
\begin{equation}\label{eq:condition-on-initial-condition-function}
     \theta_0^\phi(X) = \theta_0^\phi\left(\phi\left(\theta^\phi(X)\right)\right), \quad \forall X \in \mathbb{X},
\end{equation}
is sufficient to uniquely identify the optimal reduced dimension representation $U^{\phi, X}$.
\end{corollary}

\begin{proof}
As the oracle function~\cref{eq:convex-oracle} maps to maximally connected convex subsets of $\mathbb{U}$, each point in any subset uniquely identifies it, and thus, uniquely identifies the  optimal reduced dimension representation $U^{\phi, X}$. Moreover, as the initial condition function has the same output for the reconstruction by~\cref{eq:condition-on-initial-condition-function}, it satisfies all of the constraints of \cref{thm:unique-solution}, as required.
\end{proof}

\begin{corollary}\label{cor:newtons-method}
As the initial condition function from~\cref{cor:initial-condition-function} always defines a point on a convex subset, it is possible to make use of Newton's method~\cite{nocedal2006numerical} to find the local minimum $U^{\phi,X} \in \mathfrak{O}^\phi(X)$ through successive updates
\begin{equation}
\begin{split}
    U_{k+1} &= U_k + \Delta U_k,\quad  U_0 = \theta_0^\phi(X), \quad k = 0, \dots ,\\
    \Delta U_k &= -\left(\left.\frac{\partial^2 \!E(\phi(U), X)}{\partial U^2}\right\rvert_{U=U_k} \right)^{-1} \left.\frac{\partial \!E(\phi(U), X)}{\partial U} \right\rvert_{U = U_k},
    \end{split}
\end{equation}
and is guaranteed to converge to the minimum,
\begin{equation}
    U^{\phi, X} = \lim_{k\to\infty} U_k,
\end{equation}
by Newton-Kantorivich~\cite{ortega1968newton}.
\end{corollary}

\begin{figure}[t]
    \centering
    \scalebox{1.4}{%
    \begin{tikzpicture}[scale=1.25]

\node at (0.6,1.25) {$\mathbb{X}$};

\draw[-{Latex[open]}] (1,-0.5) -- (1, 1.5);
\draw[-{Latex[open]}] (1,-0.5) -- (4, -0.5);
\draw[-{Latex[open]}] (1,-0.5) -- (0, -1.5);

\draw [color=RoyalBlue,fill=RoyalBlue, fill opacity=0.2] plot [smooth, tension=1] coordinates {(3,0.75) (2.5,0) (2.1, -1.3)} -- plot [smooth, tension=1] coordinates {(2.1, -1.3) (2.8,-1.2) (3.5, -1.3)} -- plot [smooth, tension=1] coordinates {(3.5, -1.3) (3.9,0) (4.4, 0.75)} -- plot [smooth, tension=1] coordinates {(4.4, 0.75) (3.7,0.85) (3,0.75)};
\node[text=RoyalBlue!80!black] at (3.3,0.5) {$\widetilde{\mathbb{X}}$};

\node at (3.2, 0)[circle,fill,inner sep=1pt] {};
\node at (3.2, -0.25) {$\widetilde{X}$};

\node at (0.8, 0.5)[circle,fill,inner sep=1pt] {};
\node at (0.8, 0.25) {$X$};

\draw[color=black,-{Stealth[scale=2]}, dashed] plot [smooth, tension=1] coordinates {(0.8,0.5) (2.0, 0.75) (3.18, 0.02)};
\node at (1.7, 0.45) {$\theta^\phi\circ\phi$};

\draw[-{Latex[open]}] (5,-0.3) -- (5, 1.42);
\draw[-{Latex[open]}] (5,-0.3) -- (6.72, -0.3);

\draw[draw = none, color=gray,pattern=north west lines, fill opacity=0.2] (5,-0.3) -- (5, 1.3) -- (6.6, 1.3) -- (6.6, -0.3) -- (5, -0.3);
\node[text=gray!80!black] at (5.8,-0.6) {$\mathbb{U}$};

\draw[color=black, opacity=0.4] plot [smooth cycle, tension=1] coordinates {(5.3,0.8) (6.3, 1.0) (5.38, -0.22)};
\draw[color=black, opacity=0.3] plot [smooth cycle, tension=1] coordinates {(5.4,0.7) (6.1, 0.8) (5.38, -0.12)};
\draw[color=black, opacity=0.2] plot [smooth cycle, tension=1] coordinates {(5.5,0.6) (5.9, 0.6) (5.38, -0.02)};

\node at (5.42, 0.16)[circle,fill,inner sep=1pt] {};
\node at (5.49, -0.05) {$U^{\phi,X}$};
  
\node at (6.2, 0.9)[circle,fill,inner sep=1pt] {};
\node at (6.44, 0.85) {$U_0$};

\draw[snake=zigzag,-{Stealth[scale=2]},segment amplitude=1pt, ,segment length=4.5pt] (6.2, 0.9) -- (5.44, 0.18);

\node at (6.5, -1.2) {$U_{k+1} = U_k + \Delta U_k$};
\draw[color=black,dashed] plot [smooth, tension=1] coordinates {(5.89, 0.6) (6.6, 0.3) (6.84, -1)};

\draw[color=black,-{Stealth[scale=2]}] plot [smooth, tension=1] coordinates {(0.8,0.5) (3.5, 1.65) (6.18, 0.92)};
\node at (3.5, 1.4) {$\theta_0$};

\draw[color=black,-{Stealth[scale=2]}] plot [smooth, tension=1] coordinates {(5.4, 0.18) (4.2, 0.54) (3.22, 0.02)};
\node at (4.3, 0.3) {$\phi$};

\end{tikzpicture}%
    }
    \caption{A demonstration of the inner-workings of~\cref{thm:unique-solution} coupled with the initial condition function from~\cref{cor:initial-condition-function}. The data $X$ is first encoded with the initial condition function $\theta_0$ into the reduced dimension representation $U_0$. Subsequent updates $\Delta U_k$ with Newton's method (\cref{cor:newtons-method}) are performed to get the minimum $U^\phi$ of the local dimensionality reduction  task~\cref{eq:local-dimensionality-reduction-task}. When decoded with $\phi$, this sequence defines the reconstruction $\widetilde{X}$.}
    \label{fig:Newtons-method}
\end{figure}
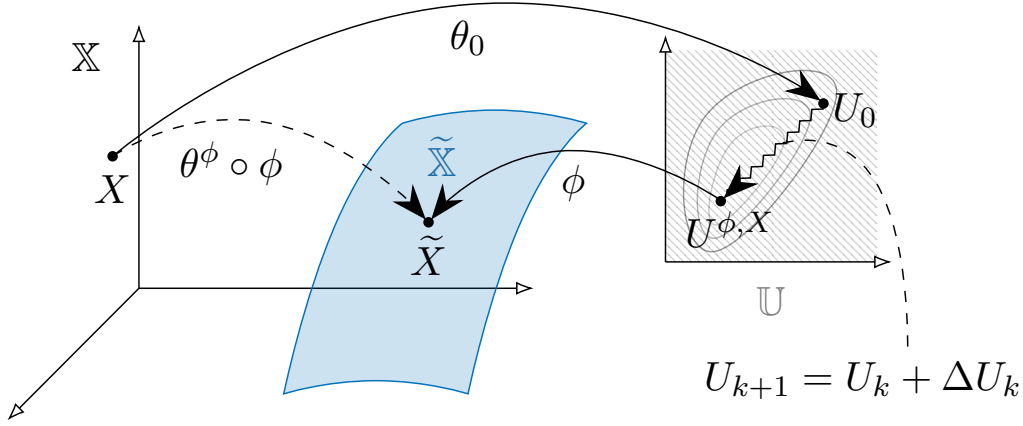

\Cref{fig:Newtons-method} illustrates how the reduced dimension representation of some $X$ is found through Newton's method (\cref{cor:newtons-method}) and decoded back to its corresponding reconstruction $\widetilde{X}$.

We now show that the solution to the dimensionality reduction task~\cref{eq:dimensionality-reduction-task} defines an encoder with the right-inverse property~\cref{eq:right-inverse-property}.

\begin{theorem}\label{thm:dimensionality-reduction-right-inverse}
Given a fixed injective decoder $\phi(\cdot)$ consider the corresponding encoder $\theta^\phi(\cdot)$ in~\cref{eq:optimal-encoder}  implicitly defined by the solution to the local dimensionality reduction task~\cref{eq:local-dimensionality-reduction-task}. Under the assumption of~\cref{thm:unique-solution}, this encoder has the right-inverse property~\cref{eq:right-inverse-property}  for any $X\in\mathbb{X}$.
\end{theorem}

\begin{proof}
Consider an arbitrary decoder $\phi$ and data point $X$. The optimal reduced dimension representation of  $X$ is defined by~\cref{eq:local-dimensionality-reduction-task}, to be $U^{\phi,X}$, then,
\begin{equation*}
    \theta^\phi\left(\phi\left(\theta^\phi\left(X\right)\right)\right) =  
    \theta^\phi(\phi(U^{\phi,X})) = \argmin_{U\, \in\, \mathfrak{O}^\phi(\widetilde{X})} \!E(\phi(U),\widetilde{X}) = U^{\phi,X} = \theta^\phi(X),
\end{equation*}
satisfying the right-inverse property~\cref{eq:right-inverse-property} as required. 
\end{proof}

\begin{corollary}\label{cor:weak-right-inverse-implied-dimensionality-reduction}
By the construction used in the proof in \cref{thm:dimensionality-reduction-right-inverse}, the right inverse property~\cref{eq:right-inverse-property} implies the optimality of the encoder for all $X$ in the image of the decoder $\widetilde{\mathbb{X}} = \phi(\mathbb{U})$.
\end{corollary}

\begin{figure}[t]
    \centering
    \begin{minipage}[c]{0.499\textwidth}
    \centering
    \scalebox{1.4}{%
    \begin{tikzpicture}[scale=1.25]
  
  \draw [color=red,fill=red, fill opacity=0.2] plot [smooth cycle, tension=1] coordinates {(1,0) (2,1) (2.5,0) (3,-1) (2,-1)};
  \node[text=red!80!black] at (2.25,-0.5) {$\widehat{\mathbb{X}}$};

    \draw [color=RoyalBlue,fill=RoyalBlue, fill opacity=0.2] plot [smooth cycle, tension=1] coordinates {(0.5,0) (1,1) (2, 0)};
    \node[text=RoyalBlue!80!black] at (0.75,0.5) {$\widetilde{\mathbb{X}}$};

    \draw[color=black,-{Stealth[scale=2]}] plot [smooth, tension=1] coordinates {(2,-1) (1.5, -1.7) (1, -1.8)};
    \node at (1.4, -1.4) {$\theta$};

    \draw [color=gray,pattern=north west lines, fill opacity=0.2] plot [smooth cycle, tension=1] coordinates {(0.5,-1.75) (0.75,-1.5) (1, -1.75) (0.75,-2)};
    \node[text=gray!80!black] at (0.75,-1.75) {$\widehat{\mathbb{U}}$};
    
    \draw[color=black,-{Stealth[scale=2]}] plot [smooth, tension=1] coordinates {(0.6,-1.55) (0.4, -1) (0.75, -0.14)};
    \node at (0.65, -1) {$\phi$};
\end{tikzpicture}%
    }%
    \end{minipage}%
    \begin{minipage}[c]{0.499\textwidth}
     \centering
    \scalebox{1.4}{%
    \begin{tikzpicture}[scale=1.25]
  
  \draw [color=red,fill=red, fill opacity=0.2] plot [smooth cycle, tension=1] coordinates {(1,0) (2,1) (2.5,0) (3,-1) (2,-1)};
  \node[text=red!80!black] at (2.25,-0.5) {$\widehat{\mathbb{X}}$};

    \draw [color=RoyalBlue,fill=RoyalBlue, fill opacity=0.2] plot [smooth cycle, tension=1] coordinates {(0.5,0) (1,1) (2, 0)};
    \node[text=RoyalBlue!80!black] at (0.75,0.5) {$\widetilde{\mathbb{X}}$};

    \draw [color=CadetBlue, dashed] plot [smooth cycle, tension=1] coordinates {(0.25,0) (0.75, 1.5)  (2.75, 1) (3, -1.5)};
    \node at (1.5,1.3) {$\textcolor{RoyalBlue!80!black}{\widetilde{\mathbb{X}}}\cup\textcolor{red!80!black}{\widehat{\mathbb{X}}}$};

    \draw[color=black,-{Stealth[scale=2]}] plot [smooth, tension=1] coordinates {(1.75,-1.25) (1.5, -1.7) (1, -1.8)};
    \node at (1.45, -1.4) {$\theta^\phi$};

    \draw [color=gray,pattern=north west lines, fill opacity=0.2] plot [smooth cycle, tension=1] coordinates {(0.5,-1.75) (0.75,-1.5) (1, -1.75) (0.75,-2)};
    \node[text=gray!80!black] at (0.75,-1.75) {$\widehat{\mathbb{U}}$};
    
    \draw[color=black,-{Stealth[scale=2]}] plot [smooth, tension=1] coordinates {(0.6,-1.55) (0.4, -1) (0.75, -0.14)};
    \node at (0.65, -1) {$\phi$};
\end{tikzpicture}%
    }%
    \end{minipage}
    \caption{A visualization of the set action of the canonical autoencoder \cref{eq:canonical-autoencoder} (left panel) as compared to the meta-autoencoder \cref{eq:meta-autoencoder} (right panel). The pre-image of $\widehat{\mathbb{U}}$ in the optimal encoder is a union of $\widehat{\mathbb{X}}$ and $\widetilde{\mathbb{X}}$. This visualizes the concepts in~\cref{thm:pre-image-correction}.}
    \label{fig:space-difference}
\end{figure}
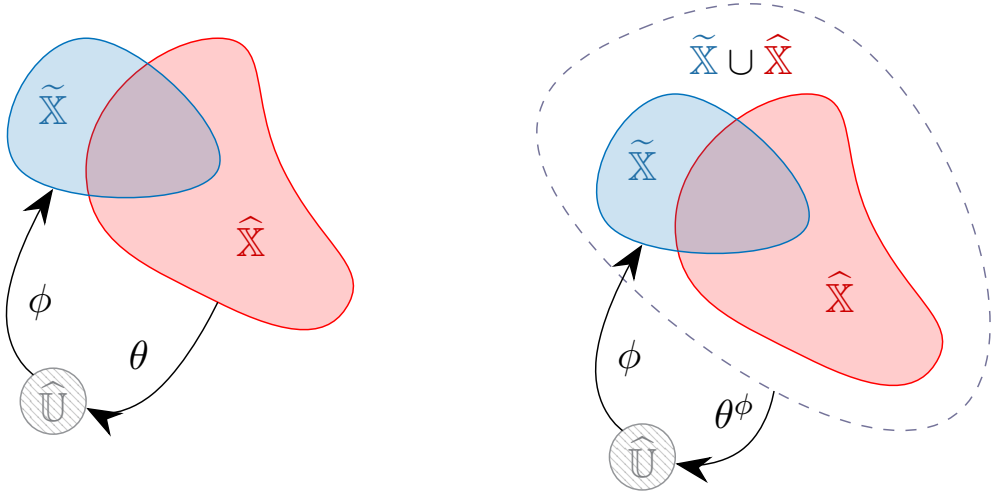

We now show that the pre-image violation~\cref{eq:pre-image-violation} does not occur for the optimal encoder~\cref{eq:optimal-encoder} defined by the local dimensionality reduction task~\cref{eq:local-dimensionality-reduction-task}.

\begin{theorem}\label{thm:pre-image-correction}
The image of $\widehat{\mathbb{U}}$ under the decoder $\phi$ is a subset of preimage of $\widehat{\mathbb{U}}$ under the optimal encoder $\theta^{\phi}$ given by~\cref{eq:optimal-encoder},
\begin{equation}
  \widetilde{\mathbb{X}} =  \phi(\widehat{\mathbb{U}}) \subset (\theta^{\phi})^{-1}(\widehat{\mathbb{U}}),
\end{equation}
when the latter is a solution to the local dimensionality reduction task~\cref{eq:local-dimensionality-reduction-task}.
This avoids the pre-image violation~\cref{eq:pre-image-violation}.
\end{theorem}

\begin{proof}
By \cref{thm:dimensionality-reduction-right-inverse}, the image of the reconstruction space $\widetilde{\mathbb{X}}$ through the encoder is the reduced dimension set $\widehat{\mathbb{U}}$,
\begin{equation*}
    \theta^\phi(\widetilde{\mathbb{X}}) = \widehat{\mathbb{U}},
\end{equation*}
and therefore,
\begin{equation}
    \widetilde{\mathbb{X}} \subset \theta^{\phi, -1}(\widehat{\mathbb{U}}).
\end{equation}
\end{proof}

An illustration of the set action of the canonical and new encoders and decoders is provided in \cref{fig:space-difference}, showing a visual representation of the concepts in~\cref{thm:pre-image-correction}.

From the proof in \cref{thm:pre-image-correction}, it is clear that every reconstruction $\widetilde{X}\in\widetilde{\mathbb{X}}$ is encodable by the the optimal encoder.

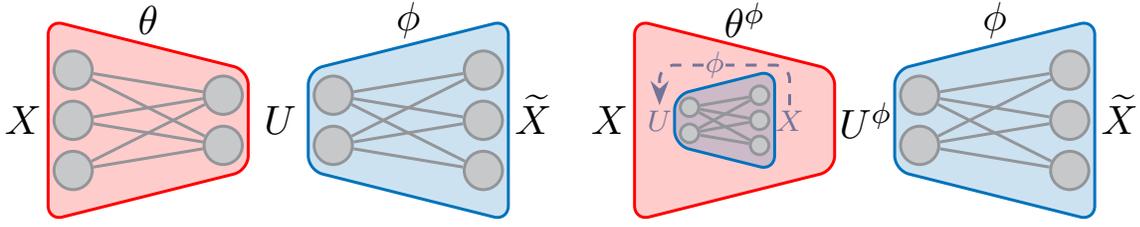
\begin{figure}[t]
    \centering
    \begin{minipage}[c]{0.499\textwidth}
    \centering
    \scalebox{1.4}{%
    \begin{tikzpicture}[scale=0.95]

\node (X) at (-0.25,-1) {$X$};
\node (X) at (2.32,-1) {$U$};
\node (X) at (4.85,-0.937) {$\widetilde{X}$};
 \path[draw, color=red, fill=red, fill opacity=0.2, rounded corners, thick] (0,0)
    -- (2, -0.5)
    -- (2,-1.5) 
    -- (0,-2) 
    -- cycle;
    
\node at (1, 0) {$\theta$};
\node at (3.6, 0) {$\phi$};
\node[draw,thick,circle,color=gray,fill=gray!50,minimum size=0.15] (Xn1) at (0.25,-0.5) {};
\node[draw,thick,circle,color=gray,fill=gray!50,minimum size=0.15] (Xn2) at (0.25,-1) {};
\node[draw,thick,circle,color=gray,fill=gray!50,minimum size=0.15] (Xn3) at (0.25,-1.5) {};

\node[draw,thick,circle,color=gray,fill=gray!50,minimum size=0.15] (Un1) at (1.75,-0.75) {};
\node[draw,thick,circle,color=gray,fill=gray!50,minimum size=0.15] (Un2) at (1.75,-1.25) {};

\draw[color=gray, thick] (Xn1) -- (Un1);
\draw[color=gray, thick] (Xn2) -- (Un1);
\draw[color=gray, thick] (Xn3) -- (Un1);

\draw[color=gray, thick] (Xn1) -- (Un2);
\draw[color=gray, thick] (Xn2) -- (Un2);
\draw[color=gray, thick] (Xn3) -- (Un2);

 \path[draw, color=RoyalBlue, fill=RoyalBlue, fill opacity=0.2, rounded corners, thick] (2.6,-0.5)
    -- (4.6, 0)
    -- (4.6,-2) 
    -- (2.6,-1.5) 
    -- cycle;

\node[draw,thick,circle,color=gray,fill=gray!50,minimum size=0.15] (Xtn1) at (4.35,-0.5) {};
\node[draw,thick,circle,color=gray,fill=gray!50,minimum size=0.15] (Xtn2) at (4.35,-1) {};
\node[draw,thick,circle,color=gray,fill=gray!50,minimum size=0.15] (Xtn3) at (4.35,-1.5) {};

\node[draw,thick,circle,color=gray,fill=gray!50,minimum size=0.15] (Utn1) at (2.85,-0.75) {};
\node[draw,thick,circle,color=gray,fill=gray!50,minimum size=0.15] (Utn2) at (2.85,-1.25) {};

\draw[color=gray, thick] (Xtn1) -- (Utn1);
\draw[color=gray, thick] (Xtn2) -- (Utn1);
\draw[color=gray, thick] (Xtn3) -- (Utn1);

\draw[color=gray, thick] (Xtn1) -- (Utn2);
\draw[color=gray, thick] (Xtn2) -- (Utn2);
\draw[color=gray, thick] (Xtn3) -- (Utn2);

\end{tikzpicture}%
    }%
    \end{minipage}%
    \begin{minipage}[c]{0.499\textwidth}
     \centering
     \scalebox{1.4}{%
    \begin{tikzpicture}[scale=0.95]

\node (X) at (-0.25,-1) {$X$};
\node (X) at (2.31,-1) {$U^\phi$};
\node (X) at (4.85,-0.937) {$\widetilde{X}$};
 \path[draw, color=red, fill=red, fill opacity=0.2, rounded corners, thick] (0,0)
    -- (2, -0.5)
    -- (2,-1.5) 
    -- (0,-2) 
    -- cycle;

\node at (1.1, 0) {$\theta^\phi$};
\node at (3.6, 0) {$\phi$};
\node[color=CadetBlue] (U) at (0.25,-1) {\scriptsize $U$};
\node[color=CadetBlue] (X) at (1.55,-1) {\scriptsize $X$};
\path[draw, color=RoyalBlue, fill=RoyalBlue, fill opacity=0.2, rounded corners, thick] (0.4,-0.75)
    -- (1.4, -0.5)
    -- (1.4,-1.5) 
    -- (0.4,-1.25) 
    -- cycle;
    

\path[dashed, color=CadetBlue,draw, rounded corners, thick] (1.55,-0.85) -- (1.55,-0.45) -- (0.91,-0.45);

\path[dashed, color=CadetBlue,draw, rounded corners, thick, -{Stealth[scale=1]}] (0.69,-0.45) -- (0.25,-0.45) -- (0.25, -0.85);
    
\node[draw,thick,circle,color=gray,fill=gray!50,inner sep=0,minimum size=5] (Un1) at (0.55,-0.87) {};
\node[draw,thick,circle,color=gray,fill=gray!50,inner sep=0,minimum size=5] (Un2) at (0.55,-1.13) {};

\node[draw,thick,circle,color=gray,fill=gray!50,inner sep=0,minimum size=5] (Xn1) at (1.25,-0.75) {};
\node[draw,thick,circle,color=gray,fill=gray!50,inner sep=0,minimum size=5] (Xn2) at (1.25,-1) {};
\node[draw,thick,circle,color=gray,fill=gray!50,inner sep=0,minimum size=5] (Xn3) at (1.25,-1.25) {};

\draw[color=gray, thick] (Xn1) -- (Un1);
\draw[color=gray, thick] (Xn2) -- (Un1);
\draw[color=gray, thick] (Xn3) -- (Un1);

\draw[color=gray, thick] (Xn1) -- (Un2);
\draw[color=gray, thick] (Xn2) -- (Un2);
\draw[color=gray, thick] (Xn3) -- (Un2);

\node[color=CadetBlue] at (0.8, -0.45) {\scriptsize $\phi$};

 \path[draw, color=RoyalBlue, fill=RoyalBlue, fill opacity=0.2, rounded corners, thick] (2.6,-0.5)
    -- (4.6, 0)
    -- (4.6,-2) 
    -- (2.6,-1.5) 
    -- cycle;

\node[draw,thick,circle,color=gray,fill=gray!50,minimum size=0.15] (Xtn1) at (4.35,-0.5) {};
\node[draw,thick,circle,color=gray,fill=gray!50,minimum size=0.15] (Xtn2) at (4.35,-1) {};
\node[draw,thick,circle,color=gray,fill=gray!50,minimum size=0.15] (Xtn3) at (4.35,-1.5) {};

\node[draw,thick,circle,color=gray,fill=gray!50,minimum size=0.15] (Utn1) at (2.85,-0.75) {};
\node[draw,thick,circle,color=gray,fill=gray!50,minimum size=0.15] (Utn2) at (2.85,-1.25) {};

\draw[color=gray, thick] (Xtn1) -- (Utn1);
\draw[color=gray, thick] (Xtn2) -- (Utn1);
\draw[color=gray, thick] (Xtn3) -- (Utn1);

\draw[color=gray, thick] (Xtn1) -- (Utn2);
\draw[color=gray, thick] (Xtn2) -- (Utn2);
\draw[color=gray, thick] (Xtn3) -- (Utn2);

\end{tikzpicture}%
    }%
    \end{minipage}
    \caption{The left panel illustrates a canonical autoencoder~\cref{eq:canonical-autoencoder} where both the encoder and the decoder are represented by neural networks. The right panel illustrates the meta-autoencoder \cref{eq:meta-autoencoder} where the encoder is defined by an optimization procedure \eqref{eq:meta-autoencoder} in terms of the decoder.%
}
    \label{fig:AE-difference}
\end{figure}

The dimensionality reduction task~\cref{eq:dimensionality-reduction-task} implicitly defines the optimal encoder~\cref{eq:optimal-encoder} for some given decoder $\phi$. In order to completely replace the canonical autoencoder \cref{eq:canonical-autoencoder}, we need a way to find the optimal decoder $\phi_*$. We can think of the problem of finding $\phi_*$ that optimally reconstructs the (local) reduced dimension representations ~\cref{eq:local-dimensionality-reduction-task} for all data $\mathbb{X}$ as a meta-task that defines a top-level optimization problem. The resulting bi-level optimization problem is:
\begin{equation}\label{eq:meta-autoencoder}
\framebox{$
\begin{gathered}
    \phi_* = \argmin_{\phi\,\in\,\*\Phi}~ \mathbb{E}_{X}\!L_\phi(\phi(\theta^\phi(X)), X),\\
    \text{s.t.}\quad \theta^\phi(X) = U^\phi = \argmin_{U\, \in\, \mathfrak{O}^\phi(\widetilde{X})} ~\!E(\phi(U), X).
\end{gathered}
$}
\end{equation}
Problem \eqref{eq:meta-autoencoder} finds the optimal decoder $\phi_*$ (and potentially initial condition function $\theta_0$ required by \cref{thm:unique-solution}), and implicitly defines the optimal encoder $\theta^\phi$ by posing the dimensionality reduction task~\cref{eq:dimensionality-reduction-task} as a constraint. The cost function $\!L_\phi$ can be, but is not necessarily equal to, the cost function $\!L$ in \cref{eq:canonical-autoencoder}.
Note that while the set of possible decoders $\*\Phi$ is unchanged from~\cref{eq:canonical-autoencoder}, the set the set of all possible encoders is not. We denote the set of all optimal encoders implicitly defined in terms of decoders in the set $\*\Phi$ by
\begin{equation}
    \*\Theta^\phi = \set*{ \theta^\phi \mid \theta^\phi(X) \text{ solves \cref{eq:local-dimensionality-reduction-task} for all $X\in\mathbb{X}$ for a } \phi \in \*\Phi}.
\end{equation}

In the context of meta-learning, the dimensionality reduction task~\cref{eq:dimensionality-reduction-task} learns  optimal representations of the data, while the meta-task~\cref{eq:meta-autoencoder} learns the optimal decoder in a statistical sense. 
We therefore name \cref{eq:meta-autoencoder} the `meta-autoencoder'.
An illustration of the neural network formulations of the canonical autoencoder and the meta-autoencoder is given in \cref{fig:AE-difference}.

We have shown that resolving one problem identified in \cref{sec:motivation}, namely explicitly accounting for the optimal reduced dimension representation~\cref{eq:reduced-dimension-graph}, resulted in recasting the autoencoder problem as a bi-level optimization~\cref{eq:meta-autoencoder} one. By including the solution to the dimensionality reduction task~\cref{eq:dimensionality-reduction-task}, two additional issues that we have identified in \cref{sec:motivation} have been addressed: the pre-image violation problem~\cref{eq:pre-image-violation} solved by~\cref{thm:pre-image-correction}, and the violation of the right-inverse property~\cref{eq:right-inverse-property} solved by \cref{thm:dimensionality-reduction-right-inverse}. We have thus layed forward a convincing argument for why the autoencoder problem should be recast as the meta-autoencoder of the form ~\cref{eq:meta-autoencoder}.

\begin{remark}
While, for the sake of exposition, we assume that the cost functions $\!L_\phi$ and $\!E$ in the bi-level optimization~\cref{eq:meta-autoencoder} are identical to the cost $\!L$ in the canonical autoencoder, that is not necessarily the case.
In fact, $\!L_\phi$ and $\!E$ can also be \textit{competing} costs. It may be the case that the cost $\!E$ attempt to find a reduced dimensional representation $U$ that preserves one optimal set of features of $X$ while the cost $\!L_\phi$ attempts to find a decoder $\phi$ that preserves an entirely different set of features. In this regard, a whole new class of non-linear dimensionality reduction algorithms is possible.
\end{remark}

\begin{remark}
Instead of solving the bi-level optimization problem presented in~\cref{eq:meta-autoencoder} in an exact fashion, an alternative is its simplification to a constrained optimization problem by the use of first-order optimality conditions similar to~\cite{cioaca2014optimization},
\begin{equation}
    \begin{gathered}
        \phi_*, U_* = \argmin_{\phi\,\in\,\*\Phi,\,U\, \in\, \mathfrak{O}^\phi(X)}~ \mathbb{E}_{X}\!L_\phi(\phi(U), X),\\
        \text{s.t.}\quad 0 = \nabla_U\!E(\phi(U), X),
\end{gathered}
\end{equation}
where the decoder and reduced order representation are solved for at the same time.
This formulation is not further explored in this work.
\end{remark}

It is clear from \cref{thm:pre-image-correction} that the meta-autoencoder~\cref{eq:meta-autoencoder} is not equivalent to the canonical autoencoder~\cref{eq:canonical-autoencoder}. It is also clear that conventional, linear methods for dimensionality reduction do not suffer from the problems that we have identified in \cref{sec:motivation}. This raises the question: under what scenario are the meta~\cref{eq:meta-autoencoder} and the canonical~\cref{eq:canonical-autoencoder} formulations of the autoencoder equivalent?
We show that when the projection and interpolation operators are taken to be linear and orthogonal, and the cost functions are taken to be mean squared error, then the canonical autoencoder~\cref{eq:canonical-autoencoder} and the meta-autoencoder~\cref{eq:meta-autoencoder} are equivalent.

\begin{theorem}\label{thm:meta-canonical-equivalence}
Consider the set of linear decoders $\*\Phi$ defined by column orthonormal matrices $\*V \in \mathbb{R}^{n\times r}$, and the set of linear encoders  $\*\Theta$ defined by row orthonormal matrices $\*W \in \mathbb{R}^{r\times n}$:
\begin{equation}\label{eq:linear-assumption}
    \phi(U) = \*V U, \quad \*V^T\,\*V=\*I_{r}; \quad \theta(X) = \*W X, \quad \*W\,\*W^T=\*I_{r}.
\end{equation}
Additionally, consider mean squared error cost functions:
\begin{equation}\label{eq:mse-assumption}
    \!L(\widetilde{X},X) = \!L_\phi(\widetilde{X},X) = \!E(\widetilde{X},X) = \frac{1}{n}\norm{\widetilde{X} - X}_2^2.
\end{equation}
Under the assumptions that both the encoder and decoder are linear \eqref{eq:linear-assumption} and all cost functions are of the form \eqref{eq:mse-assumption}, the solutions to the canonical autoencoder problem~\cref{eq:canonical-autoencoder} and the meta-autoencoder~\cref{eq:meta-autoencoder} are equivalent under vector-wise sign change.
\end{theorem}

\begin{proof}
We first examine the canonical autoencoder~\cref{eq:canonical-autoencoder}. Under the given assumptions, the problem,
\begin{equation*}
    \*V_*, \*W_* = \argmin_{\*V\,\in\,\*\Phi, \*V\,\in\,\*\Theta} \mathbb{E}_{X}\frac{1}{n}\norm{\*V\*W X - X}_2^2,
\end{equation*}
reduces to finding the rank $r$ matrix $\*V_*\*W_*$ that defines the optimal subspace of the reconstruction $\widetilde{\mathbb{X}}$. 
This problem is equivalent to finding the $r$ dominant eigenvectors of the covariance  of $X$,
\begin{equation*}
    \*V_*\*\Lambda_r\*W_* =  \*V_*\*\Lambda_r\*V_*^T \approx \mathbb{E}_X X X^T,
\end{equation*}
where $\*W_* = \*V_*^T$, by the optimality of the Karhunen-Lo\`eve expansion~\cite{smith2013uncertainty}.

We next examine the local dimensionality reduction task~\cref{eq:local-dimensionality-reduction-task} in the linear orthogonal case for some decoder , under a least-squares cost function,
\begin{equation*}
    U^{\*V, X} = \argmin_{U\, \in\, \mathfrak{O}^{\*V}(X) } \frac{1}{n} \norm{\*V U - X}_2^2.
\end{equation*}
We note that the problem is convex over all of $\mathbb{R}^r$, thus the oracle function is $\mathfrak{O}^{\*V}(X) = \mathbb{R}^r$ trivially for any $X$. It is therefore the case that the dimensionality reduction task~\cref{eq:local-dimensionality-reduction-task}  simply reduces to a linear least-squares problem, with the solution
\begin{equation*}
     U^{\*V, X} = \*V^T\, X.
\end{equation*}
We now focus our attention to the meta-task of finding the optimal decoder $\*V$ under a least-squares cost function,
\begin{equation*}
    \*V_* = \argmin_{\*V\,\in\,\*\Phi} \mathbb{E}_{X}\frac{1}{n}\norm{\*V\*V^T X - X}_2^2,
\end{equation*}
which again reduces to finding the $r$ dominant eigenvectors of the covariance  of $X$,
\begin{equation}
    \*V_*\*\Lambda_r\*V_*^T \approx \mathbb{E}_X X X^T.
\end{equation}
As the $r$ dominant eigenvectors are unique modulo a sign change, the two problem formulations produce identical results under the given assumptions, as required.
\end{proof}

The proof of \cref{thm:meta-canonical-equivalence} shows that in the orthogonal linear case with mean squared error cost, the local dimensionality reduction task~\cref{eq:local-dimensionality-reduction-task} has the same result as the meta-autoencoder problem \eqref{eq:meta-autoencoder}. In this scenario there is no need to explicitly split the dimensionality reduction problem as in \eqref{eq:meta-autoencoder}.

\begin{remark}
It is interesting to note that there are two requirements for~\cref{thm:meta-canonical-equivalence} to be valid. One is the linearity and orthogonality of the operators~\cref{eq:linear-assumption}, the other is that both cost functions $\!L_\phi$ and $\!E$ in the meta-autoencoder~\cref{eq:meta-autoencoder} are equal and of mean squared error form~\cref{eq:mse-assumption}. Giving up the orthogonality and the equality the cost functions assumptions may lead to linear dimensionality reduction techniques different from existing formulations.
\end{remark}

\section{Implementation of Meta-Autoencoders}
\label{sec:meta-auto-encoder-issues}

We now turn our attention towards implementing the meta-autoencoder through neural-network-based representations. Assume that the set $\*\Phi$ in \cref{eq:meta-autoencoder} is some family of neural networks parameterized by weights $\*W^\phi$.
In this scenario, a full solution to the meta-autoencoder problem~\cref{eq:meta-autoencoder} would involve finding said weights by a full solution to the bi-level optimization problem, through the use of the optimal encoder~\cref{eq:optimal-encoder}.
Therefore, there are two potentially large computational costs that need to be addressed by a practical implementation of the meta-autoencoder:
(i) the cost of solving the local dimensionality reduction task~\cref{eq:local-dimensionality-reduction-task} repeatedly due to the bi-level form of the optimization \cref{eq:meta-autoencoder}, and (ii) the cost of training the meta-autoencoder, which can be significantly more expensive than training the canonical autoencoder~\cref{eq:canonical-autoencoder}.
We discuss two ways to bootstrap existing network evaluation and training procedures to alleviate these costs.

\subsection{Autoencoder Correction}
\label{sec:auto-encoder-correction}

\begin{algorithm}[t]
\caption{Autoencoder correction}
\label{alg:encoder-corrected}
\begin{algorithmic}[1]
\STATE{Input: set of data points $X\in\widehat{\mathbb{X}}$, set of neural networks for the encoder, $\*\Theta$, and for the decoder, $\*\Phi$.}
\STATE{Output: the corrected encoder $\theta^{\phi}$ and decoder $\phi$.}
\STATE{\% Find an encoder and decoder pair that satisfies~\cref{eq:canonical-autoencoder}}
\STATE{$\phi_*,\,\theta_*  \xleftarrow{} \argmin_{\phi\,\in\*\Phi,\,\theta\,\in\*\Theta} \mathbb{E}_{X}\!L\big(\phi\left(\theta(X)\right), X\big)$}
\STATE{\% Define the encoder $\theta_*$ to be the initial condition function from \cref{cor:initial-condition-function}}
\STATE{$\theta_0^\phi \coloneqq \theta_*$}
\STATE{\% Construct the corrected encoder $\theta^{\phi_*}$  by solving the local dimensionality reduction task~\cref{eq:local-dimensionality-reduction-task} }
\STATE{$\theta^{\phi_*}(X) \coloneqq \argmin_{\substack{U\, \in\, \mathfrak{O}^\phi_*(X)\\U_0 = \theta_0^\phi(X)} } \!E(\phi_*(U), X)$}
\RETURN $\theta^\phi = \theta^{\phi_*}, \phi = \phi_*$
\end{algorithmic}
\end{algorithm}

We first consider solving challenge (i) where the cost of training an autoencoder is the biggest bottleneck. In this scenario we are not concerned with challenge (ii) where evaluating the resulting autoencoder is more expensive than evaluating the canonical autoencoder~\cref{eq:canonical-autoencoder}.

We introduce a strategy by which existing canonical autoencoders~\cref{eq:canonical-autoencoder} can be corrected through a post-processing step. 
The procedure, described by algorithm~\cref{alg:encoder-corrected} is as follows: first, train an encoder and decoder pair through the canonical autoencoder~\cref{eq:canonical-autoencoder}. Second, define the encoder to be the initial condition function. Finally, define the corrected encoder by solving the local dimensionality reduction task~\cref{eq:local-dimensionality-reduction-task}.

The procedure above does not solve the full meta-autoencoder problem~\cref{eq:meta-autoencoder}, but instead only aims to solve the local dimensionality reduction task~\cref{eq:local-dimensionality-reduction-task}.
There are no guarantees, however that this procedure will solve any of the issues identified in~\cref{sec:motivation}.
From a practical point of view, this correction can be readily applied to existing autoencoders, without a full retraining procedure, or can be used as a post-processing step in most operational workflows.

\subsubsection{Right-invertible Autoencoders}

We next look at the opposite end of the spectrum, where training or retraining an autoencoder is not a significant factor, but the cost of evaluating the autoencoder dominates. In this scenario, we wish to eliminate the definition of the implicit encoder~\cref{eq:optimal-encoder} by eliminating the local dimensionality reduction task~\cref{eq:local-dimensionality-reduction-task} from the meta-autoencoder~\cref{eq:meta-autoencoder} formulation, and replacing it with a condition that is simpler to satisfy. 

By eliminating the implicit definition of the encoder, we reintroduce the set $\*\Theta$ of all possible encoders from the canonical autoencoder~\cref{eq:canonical-autoencoder}, to the already existing set of all decoders $\*\Phi$.
Taking the results from \cref{thm:dimensionality-reduction-right-inverse} and  \cref{cor:weak-right-inverse-implied-dimensionality-reduction}, it is clear that enforcing the right-inverse property~\cref{eq:right-inverse-property} is a weaker form of enforcing the local dimensionality reduction task~\cref{eq:local-dimensionality-reduction-task}. 
We can therefore replace the optimization constraint in the meta-autoencoder problem~\cref{eq:meta-autoencoder} with the enforcement of the right-inverse property~\cref{eq:right-inverse-property} over the data in the canonical autoencoder~\cref{eq:canonical-autoencoder}:
\begin{equation}\label{eq:right-invertible-autoencoder}
\begin{gathered}
    \phi_*,\, \theta_* = \argmin_{\phi\,\in\*\Phi,\, \theta\,\in\*\Theta} \mathbb{E}_{X}\!L_\phi(\phi\big(\theta(X)\big), X)\\
    \text{s.t.}\quad \theta\left(\phi\left(\theta\left(X\right)\right)\right) =  \theta\left(X\right),\quad \forall X \in \widehat{\mathbb{X}}. 
\end{gathered}
\end{equation}
We call the resulting formulation~{eq:right-invertible-autoencoder} the \textit{right-invertible autoencoder}. 

The formulation~\cref{eq:right-invertible-autoencoder} has several potential issues. First, and foremost, it is possible that the constraint cannot be satisfied by the chosen sets $\*\Phi$ and $\*\Theta$ of decoders and encoders, i.e., there is no feasible solution satisfying the constraints, as practical neural-networks are universal approximators~\cite{hornik1989multilayer} up to a certain measure of error.
A second, more subtle problem, is that $\phi\big(\theta(X)\big)$ might not be encodable by all the encoders $\theta\in\*\Theta$, thereby preventing the solution of~\cref{eq:right-invertible-autoencoder} through practical means. 

If we restrict our sets of decoders $\*\Phi$ and encoders $\*\Theta$ to some classical neural network architectures (such as dense feed-forward neural networks), then the problem of encodability can be hand-waved away. However, the problem of satisfying the constraint comes to the forefront as a prominent issue.
As finite neural networks are universal function approximators~\cite{aggarwal2018neural}, up to some certain error, there may exist functions in the chosen sets $\*\Phi$ and $\*\Theta$ of decoders and encoders that approximate the enforcement of the right-inverse property in~\cref{eq:right-invertible-autoencoder} in a weak manner as is typical in `physics-informed' machine learning approaches~\cite{willard2020integrating}.

Given some cost $\!L_R : \mathbb{U} \times \mathbb{U} \to \mathbb{R}$, a weak-constrained formulation of~\cref{eq:right-invertible-autoencoder} can be constructed,
\begin{equation}\label{eq:weak-right-invertible-autoencoder}
    \phi_*,\, \theta_* = \argmin_{\phi\,\in\*\Phi,\, \theta\,\in\*\Theta} \mathbb{E}_{X}\left[\!L_\phi(\phi\big(\theta(X)\big), X) + \lambda\!L_R\left(\theta\left(\phi\left(\theta\left(X\right)\right)\right) , \theta\left(X\right)\right) \right],
\end{equation}
parameterized by the hyper-parameter $\lambda$. This formulation is inexpensive to implement and is a natural extension of the canonical autoencoder formulation~\cref{eq:canonical-autoencoder}. It has been utilized for model order reduction in \cite{popov2021nlmfenkf}.

\begin{remark}
It is of independent interest to derive neural network architectures that satisfy exactly the right-inverse property.
\end{remark}

\begin{remark}
Assuming that $\!L_R \geq 0$ over all possible inputs, an alternative formulation to \cref{eq:weak-right-invertible-autoencoder} would be through an inequality constrained formulation,
\begin{equation}\label{eq:inequality-right-invertible-autoencoder}
\begin{gathered}
    \phi_*,\, \theta_* = \argmin_{\phi\,\in\*\Phi,\, \theta\,\in\*\Theta} \mathbb{E}_{X}\!L_\phi(\phi\big(\theta(X)\big), X)\\
    \text{s.t.}\quad \!L_R\left(\theta\left(\phi\left(\theta\left(X\right)\right)\right), \theta\left(X\right) \right) \le \epsilon,\quad \forall X \in \mathbb{X}. 
\end{gathered}
\end{equation}
where the right invertibility constraint is preserved up to some threshold parameter $\epsilon$.
\end{remark}

\subsection{Practical Implementation}

\begin{algorithm}[t]
\caption{Optimal encoder $\theta^\phi$ through  a non-linear least squares Gauss-Newton method}
\label{alg:encoder-practical}
\begin{algorithmic}[1]
\STATE{Input: sample $X$, initial encoder $\theta_0^\phi$, decoder $\phi$, number of Newton iterations $I$}
\STATE{Output: the meta-encoded representation $\theta^\phi(X) = U^{\phi,X}$}
\STATE{\% Get the initial condition for the Newton iterations}
\STATE{$U_0 = \theta_0^\phi(X)$}
\FORALL{newton iterations $i$ from $1$ to $I$}
\STATE{\% Calculate the reconstruction}
\STATE{$\widetilde{X} \xleftarrow{}{} \phi(U_{i-1})$}
\STATE{\% Evaluate the reconstruction mismatch}
\STATE{$F \xleftarrow{} \widetilde{X} - X$}
\STATE{\% Evaluate the Jacobian by automatic differentiation}
\STATE{$J \xleftarrow{} \big(\left.\nabla_U F\right|_{U_{i-1}}\big)^T$}
\STATE{\% Solve the sub-problem for the Newton direction}
\STATE{$p \xleftarrow{} (J^T J)^{-1} J^T F$}\label{line:linear-solve}
\STATE{\% update the reduced state estimate}
\STATE{$U_i \xleftarrow{} U_{i-1} - p$}
\ENDFOR
\RETURN $U^{\phi,X} = U_I$
\end{algorithmic}
\end{algorithm}

We provide a practical implementation of the optimal encoder~\cref{eq:optimal-encoder} that does not address the issues outlined in \cref{sec:meta-auto-encoder-issues}, namely that the cost to train and the cost to evaluate the encoder are increased, however, it is guaranteed to address the problems identified in \cref{sec:motivation}.

In order to solve the meta-autoencoder problem~\cref{eq:meta-autoencoder} we will be employing a method similar to the one discussed in~\cite{ji2021bilevel}. Instead of treating the local dimensionality reduction task~\cref{eq:local-dimensionality-reduction-task} as a separate optimization problem, we treat it as an encoder~\cref{eq:optimal-encoder} with the optimization algorithm embedded into its formulation. In order to find gradients of~\cref{eq:meta-autoencoder} we back-propagate \textit{through another optimization algorithm}. This is similar to how many PDE-constrained optimization algorithms such as 4D-Var~\cite{asch2016data} operate. 

\begin{remark}
Applying existing methods for solving stochastic-deterministic bi-level optimization problems~\cite{shisheng2021complexity} to the meta-autoencoder problem~\cref{eq:meta-autoencoder} is of independent interest.
\end{remark}

Using the machinery discussed in~\cref{cor:initial-condition-function}, there are two functions that must be defined: the decoder $\phi$, and the initial condition $\theta_0$. We take both of these  functions to be implemented by neural networks. We denote the set of all possible decoders of a given neural network architecture by $\*\Phi$. Similarly we assume that the set of all possible initial condition functions of a given neural network architecture is given by the set  $\*\Theta^0$. We assume that both these sets can be parameterized by some collection of weights, which we call $W^\phi$ for the weights of the decoder and $W^{\theta_0}$ for the weights of the initial condition function.
For the sake of simplicity we will ignore the condition~\cref{eq:condition-on-initial-condition-function}.

For the rest of this section, we will assume that the dimensionality reduction task cost function is the least squares cost function~\cref{eq:mse-assumption} which is amenable to solution by the Gauss-Newton algorithm. One possible implementation of the optimal encoder through the Gauss-Newton method is given in~\cref{alg:encoder-practical}.

\begin{remark}
The Gauss-Newton algorithm given in~\cref{alg:encoder-practical} is only practical for relatively small reduced dimension sizes $r$. Computing the Jacobian $J$ can potentially be expensive for many neural-network architectures. 
Additionally, linear system solves can have their own issues such as negative-definite matrices or computational intensity.
The linear solve in \cref{line:linear-solve} of \cref{alg:encoder-practical} could be performed so as to match the Levenberg–Marquardt algorithm~\cite{nocedal2006numerical},
\begin{equation}
    p \xleftarrow{} (\gamma I_r + J^T J)^{-1} J^T F,
\end{equation}
where the factor $\gamma$ is a damping factor.
Methods such as L-BFGS~\cite{liu1989limited_lbfgs} can also mitigate many of said challenges. Integrating these ideas into a performant implementation is outside the scope of this paper. 
\end{remark}

In practice we only have access to a finite subset of $N$ data points $\{X_1, X_2, \dots, X_N\} \subset\mathbb{X}$, which are assumed to be exchangeable samples from the underlying distribution.
The practical optimization problem reads
\begin{equation}\label{eq:practical-optimization-problem}
\begin{gathered}
     W^\phi_*,\, W^{\theta_0}_* = \argmin_{W^\phi,\, W^{\theta_0}} \frac{1}{N} \sum_{i=1}^N\!L_\phi(\phi(U^\phi_i), X_i),\\
    \text{s.t.}\,\, U^\phi_i = \theta^\phi(X_i),
    \end{gathered}
\end{equation}
where $\theta^\phi$ is defined in terms of $\theta_0$ which in turn is parametrized by $W^{\theta_0}$.

As stated before, when the optimization problem~\cref{eq:practical-optimization-problem} is solved with conventional gradient based stochastic optimization methods such as stochastic gradient descent, the back-propagation is carried out  through the Gauss-Newton iterations. This is an additional challenge for the practical solution of the meta-autoencoder problem~\cref{eq:meta-autoencoder} and a costly barrier to its practical implementation.

\section{Numerical Illustration}
\label{sec:numerical-illustration}

\begin{figure}[t]
    \centering
    \begin{tikzpicture}
        \node (sixes) at (0,0) {\includegraphics[width=0.65\linewidth]{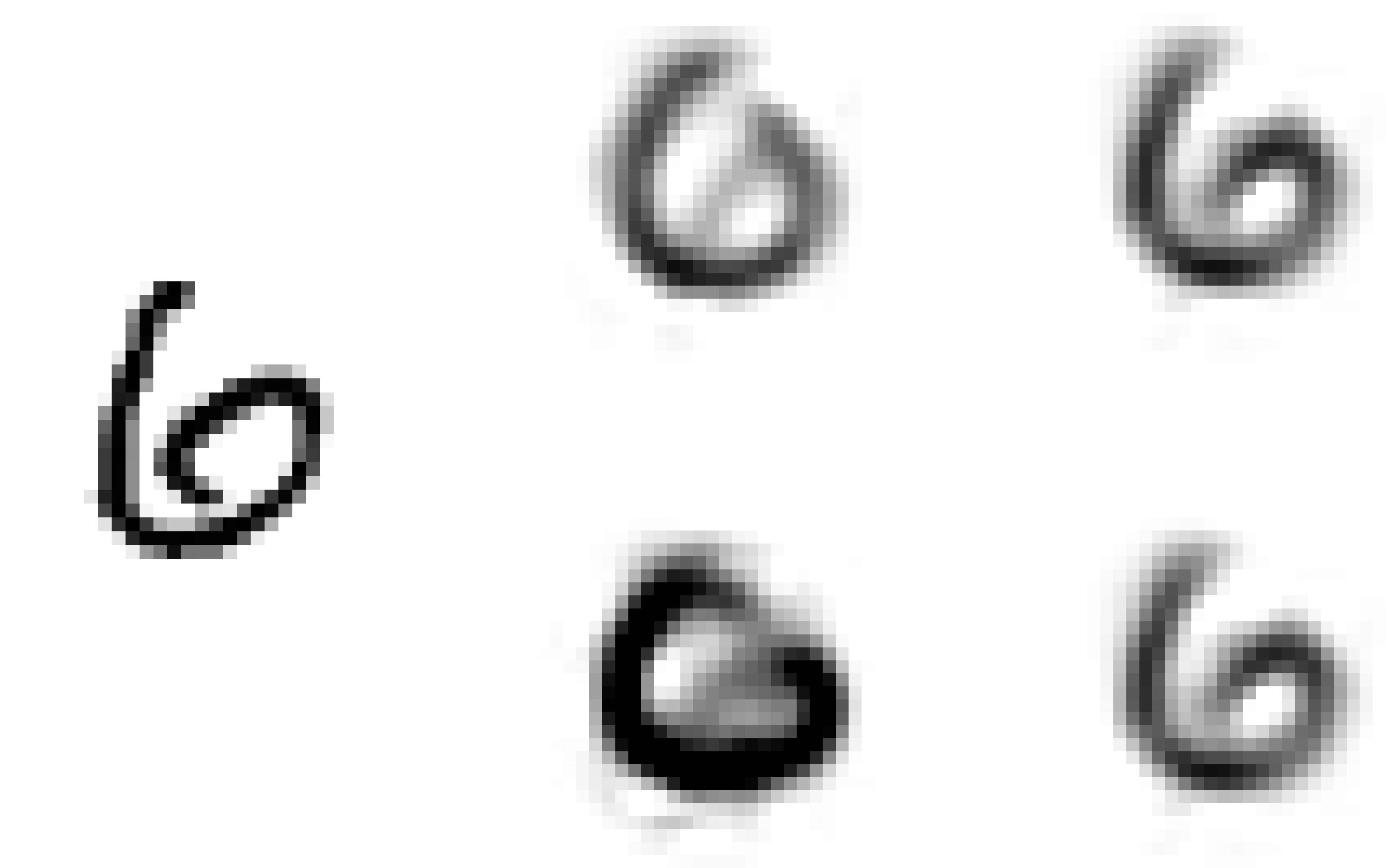}};%
        \node (test) at (-3.6, 1.7) {\large\textbf{\textsf{(i) Original}}};
        \node (test) at (0, 3.5) {\large\textbf{\textsf{(ii) Canonical}}};
        \node (test) at (3.75, 3.5) {\large\textbf{\textsf{(iii) Meta}}};
        \node (test) at (7.5, 2.1) {\large\textbf{\textsf{First Reconstruction}}};
        \node (test) at (7.5, -1.6) {\large\textbf{\textsf{100th Reconstruction}}};
    \end{tikzpicture}
   
    \caption{From left to right: (i) the original handwritten sample of six, (ii) a canonical autoencoder \cref{eq:canonical-autoencoder} reconstruction with its 100th reconstruction, and (iii) the reconstruction obtained by the meta-autoencoder \cref{eq:meta-autoencoder} with its 100th reconstruction.}
    \label{fig:visual-inspection}
\end{figure}
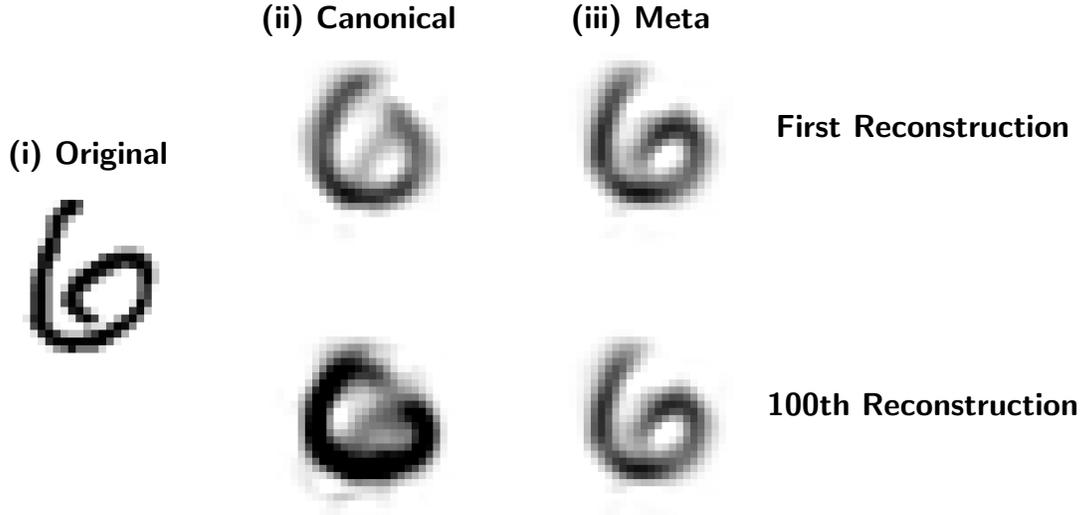

\begin{table}[]
    \centering
    \begin{tabular}{|c|c|c|c|}
    \hline
              & Canonical Autoencoder & meta-autoencoder     \\\hline
      Training MSE  & 2.8594e-02  & \textbf{2.6148e-02} \\\hline
      Testing MSE  & 2.9126e-02  & \textbf{2.6416e-02} \\\hline
      Training Second Reconstruction MSE & 1.3154e-03 & \textbf{3.1645e-05}\\\hline
      Testing Second Reconstruction MSE & 1.3101e-03 &  \textbf{3.2502e-05}\\\hline
    \end{tabular}
    \caption{Reconstruction and second reconstruction mean squared error over the MNIST training and testing data for the canonical autoencoder~\cref{eq:canonical-autoencoder} and meta-autoencoder~\cref{eq:meta-autoencoder}.}
    \label{tab:testing-error}
\end{table}

The example in~\cref{eq:right-inverse-violation-example} of a worst-case scenario where a six is misidentified as a triangle, generally is not so obvious in real-world applications of autoencoders. If it did, the issues that we have identified in this paper would have been noticed sooner. The real danger is that these issues often manifest in very subtle ways.
We show an example of an obvious manifestation of these issues using the MNIST \cite{deng2012mnist} dataset of handwritten samples of digits. 

We focus on the canonical autoencoder~\cref{eq:canonical-autoencoder}, and the meta-autoencoder~\cref{eq:meta-autoencoder}.

For the encoder~\cref{eq:encoder} (initial condition function~\cref{eq:encoder-initial-condition} in the meta-autoencoder case) and decoder~\cref{eq:decoder} architectures we take a simple one hidden layer feed-forward architecture,
\begin{equation}
    \begin{gathered}
        \theta_0(X),\, \theta(X) = W_2^\theta \tanh(W_1^\theta X + b_1^\theta) + b_2^\theta,\\
        W_1^\theta \in \mathbb{R}^{h\times n},\,\, W_2^\theta \in \mathbb{R}^{r\times h},\,\,b_1^\theta \in \mathbb{R}^{h},\,\, b_2^\theta \in \mathbb{R}^{r},\\
        \phi(U) = W_2^\phi \tanh(W_1^\phi U + b_1^\phi) + b_2^\phi,\\
        W_1^\phi \in \mathbb{R}^{h\times r},\,\, W_2^\phi \in \mathbb{R}^{n\times h},\,\,b_1^\phi \in \mathbb{R}^{h},\,\, b_2^\phi \in \mathbb{R}^{n},\\
    \end{gathered}
\end{equation}
with the $\tanh$ non-linearity in order to keep the encoder and decoder continuous functions. The hidden layer dimension we set to $h=100$. We deliberately do not tune this hyperparameter because we want to show a degenerate case.

In dimensionality reduction there is a delicate balance between accuracy and efficiency. On one hand, the reduced dimension $r$ has to be large enough to accurately represent the key features of the data. On the other hand, $r$ has to be small enough in order for a computational advantage to exist. 
In many traditional machine learning applications, $r$ is treated as a hyperparameter, as the desired level of accuracy can be achieved with a small $r$. Many scientific applications are resource-bound, and thus there is typically a fixed budget for computational cost, thus $r$ has to be of a fixed size. For this reason, we assume herein the small value of $r=5$.

We similarly treat the number of epochs, batch size, and learning-rate scheduler as constants sufficient enough to achieve a global minimum, and not as hyperparameters that need tuning.

The two cost functions in \cref{eq:meta-autoencoder} that we take are the mean squared errors~\cref{eq:mse-assumption}.
For the optimal encoder~\cref{eq:optimal-encoder}, we perform a fixed number of $I=4$ Gauss-Newton iterations, and the linear system solve is carried out with the conjugate-gradient method. 
Both the the canonical and meta-autoencoders are trained with ADAM~\cite{kingma2014adam} optimization procedure. 

\begin{remark}
From empirical observations, more Gauss-Newton iterations led to challenges such as vanishing gradient. Fixing the number of iterations is equivalent to regularizing the optimization problem. It is of independent interest to stabilize the Gauss-Newton iterations in the optimal encoder, to be both accurate in the forward and adjoint modes.
\end{remark}

For testing we use the MNIST testing data set. We look at two measures of accuracy, the mean squared error~\cref{eq:mse-assumption} and the second reconstruction mean squared error,
\begin{equation}\label{eq:second-reconstruction-error}
    \mathrm{SR\,MSE}(X) = \frac{1}{n N}\norm{\phi\big(\theta\big(\phi\big(\theta(X)\big)\big)\big) - \phi\big(\theta(X)\big)}_F^2
\end{equation}
where $N$ stands for the number of data samples. We use the reconstruction of the right-inverse property to have a comparable scale and avoid violations of the assumptions in \cref{thm:unique-solution}.

\Cref{tab:testing-error} showcases the testing mean squared error over the whole testing set. Over the testing set, the meta-autoencoder shows about a $9\%$ decrease in reconstruction error, and a $98\%$ decrease in the second reconstruction error. The second of the two results strongly suggest that the practical implementation of the meta-autoencoder reduces violations of the right-inverse property. 

The visual inspection in \cref{fig:visual-inspection} of the  first digit `6' from the MNIST testing data-set showcases different reconstructions with respect to each autoencoder tested. The canonical autoencoder recognizes this digit as something between a zero, a six and an eight, while the meta-autoencoder correctly encoded and decoded a six.
We also showcase the 100th reconstruction of each, similar in the way that~\cref{eq:second-reconstruction-error} is computed. It is clear that degenerate behavior occurs for the canonical autoencoder, while the meta-autoencoder is completely stable.

The results strongly suggest that the meta-autoencoder~\cref{eq:meta-autoencoder} formulation is significantly different than the canonical autoencoder~\cref{eq:canonical-autoencoder} formulation, and that it produces more accurate encoder-decoder pairs.

\section{Conclusions}
\label{sec:conclusions}

In this paper we show that the canonical formulation of the autoencoder problem~\cref{eq:canonical-autoencoder} suffers from issues that can limit its ability to solve applications such as reduced order modeling.
We provide a formulation~\cref{eq:meta-autoencoder} based on meta-learning that accounts for the dimensionality reduction task~\cref{eq:dimensionality-reduction-task},
and provide proofs that it does not suffer from the same issues.
A numerical illustration on the MNIST data set, shows significant reductions in error and visually showcases the stability of the meta-autoencoder approach.

The meta-autoencoder framework leads to well-posed formulations of the encoder and decoder. Unlike weak heuristics such as regularization and dropout, the dimensionality reduction task is a deterministic constraint and provides a provably optimal solution to the autoencoder problem. 

We show that under certain simplifying assumptions the meta-autoencoder formulation provides a unique reduced dimension representation of any given data point. 
We provide a simple practical implementation to solve  problem for a least-squares cost function, where the the full representation  of the encoder is given by a neural network coupled with an optimization algorithm that minimizes an internal cost function.

Future research will focus on the use of the meta-autoencoders for performing model order reduction, and other data-driven scientific tasks. Practical approaches such as posing the bi-level optimization in~\cref{eq:meta-autoencoder} in a weaker form~\cite{bard2013practical, cioaca2014optimization} are also of great interest.

Additionally, the authors conjecture that almost all results that involve dimensionality reduction with autoencoders have to be re-evaluated with the ideas presented herein into account.

\bibliographystyle{siamplain}
\bibliography{biblio}

\end{document}